\documentclass{article}

\usepackage[margin=1in]{geometry}
\usepackage[utf8]{inputenc} 
\usepackage[T1]{fontenc}    
\usepackage{hyperref}       
\usepackage{url}            
\usepackage{booktabs}       
\usepackage{amsfonts}       
\usepackage{nicefrac}     
\usepackage{microtype}      
\usepackage{xcolor}         
\usepackage{amsmath}

\newcommand{\distr}{P}
\newcommand{\score}{s_t}
\newcommand{\cells}{\mathcal{C}_{\mathcal{F}}}
\newcommand{\Gt}{t}
\newcommand{\Bopt}{t^\alpha_{P,F}}

\newcommand{\featureset}{\mathcal{F}}

\newcommand{\FairObj}{L_P^{\mbox{fair}}}
\newcommand{\DPObj}{L_P^{\mbox{DP}}}
\newcommand{\EOObj}{L_P^{\mbox{EO}}}

\newcommand{\FairObjective}{L^{\mbox{fair}}}
\newcommand{\DPObjective}{L^{\mbox{DP}}}
\newcommand{\EOObjective}{L^{\mbox{EO}}}
\newcommand{\PredObjective}{L^{\mbox{Pred}}}

\newtheorem{defn}{Definition}
\newtheorem{thm}{Theorem}
\newtheorem{cor}{Corollary}

\newtheorem{lem}{Lemma}

\newtheorem{clm}{Claim}

\newenvironment{proof}{\noindent{\bf Proof:~~}}{\(\square\)}
\newenvironment{proofof}[1]{\noindent{\bf Proof of {#1}:~~}}{\(\square\)}
\newcommand{\BPFOF}{\begin{proofof}} \newcommand {\EPFOF}{\end{proofof}}

\newcommand{\indep}{\perp \!\!\! \perp}

\newcommand{\Hcal}{{\mathcal H}}

  \newcommand{\ex}{\mathbb{E}}

\renewcommand{\featureset}{\mathcal{F}}

\title{Impossibility results for fair representations}

\author{
  Tosca Lechner\footnote{\texttt{tosca.lechner@gmail.com}}, Shai Ben-David\footnote{\texttt{bendavid.shai@gmail.com}}\thanks{Also member of the Vector Institute, Toronto.},  Sushant Agarwal\footnote{\texttt{sushantagarwal96@gmail.com}}, Nivasini Anathakrishnan\footnote{\texttt{nanathakrishnan@uwaterloo.ca}} \\
  School of Computer Science,
  University of Waterloo\\
  Waterloo, Ontario, Canada \\
  
}

\begin{document}

\maketitle

\begin{abstract}
 With the growing awareness to fairness in machine learning and the realization of the central role that data representation has in data processing tasks, there is an obvious interest in notions of fair data representations. The goal of such representations is that a model trained on data under the representation (e.g., a classifier) will be guaranteed to respect some fairness constraints.
 Such representations are useful when they can be fixed for training models on various different tasks and also when they serve as data filtering between the raw data (known to the representation designer) and potentially malicious agents that use the data under the representation to learn predictive models and make decisions.
 A long list of recent research papers strive to provide tools for achieving these goals.
 
 However, we prove that this is basically a futile effort. Roughly stated, we prove that no representation can guarantee the fairness of classifiers for different tasks trained using it; even the basic goal of achieving label-independent Demographic Parity fairness fails once the marginal  data distribution shifts. More refined notions of fairness, like Odds Equality, cannot be guaranteed by a representation that does not take into account the task specific labeling rule with respect to which such fairness will be evaluated (even if the marginal data distribution is known a priory). Furthermore, except for trivial cases, no representation can guarantee Odds Equality fairness for any two different tasks, while allowing accurate label predictions for both.  
 
While some of our conclusions are intuitive, we formulate (and prove) crisp statements of such impossibilities, often contrasting impressions conveyed by many recent works on fair representations. 

\end{abstract}

\section{Introduction}
Automated decision making has become more and more successful over the last few decades and has therefore been used in an increasing number of domains, either as stand alone, or to support human decision makers. This includes many sensitive domains which significantly impact people's livelihoods, such as granting loans, university admissions, recidivism predictions, or insurance rate settings. It has been found that many such decision tools, often unintentionally, have biases against minority groups, and therefore lead to discrimination. In response to these concerns, the machine learning research community has been devoting effort to developing clear notions of fair decision making, and coming up with algorithms for implementing fair machine learning.\\

A common approach to address the important issue of fair algorithmic decision making is through \emph{fair data representation}. The idea is that some regulator, or a responsible data curator, transforms collected data to a format (-- \emph{representation}), that can then be used for solving downstream classification tasks, while providing guarantees of fairness. This approach was put forward by the seminal paper of Zemel et al. \cite{zemel2013learning}. In their words: "our intermediate representation can be used for other classification tasks (i.e., transfer learning is possible)"... "We further posit that such an intermediate representation is
fundamental to progress in fairness in classification,
since it is composable and not ad hoc; once such a
representation is established, it can be used in a blackbox fashion to turn any classification algorithm into a
fair classifier, by simply applying the classifier to the
sanitized representation of the data". Many followup papers aim to realize this paradigm, solving technical and algorithmic issues \cite{pmlr-v80-madras18a, DBLP:journals/corr/EdwardsS15, mcnamara2019costs, song2019learning, DBLP:conf/icml/zemmel2019} (to mention just a few). The main contribution of this paper is showing that, basically,
\emph{it is impossible to achieve this goal}. Namely, for Demographic Parity (DP) fairness, given any domain partitioned into two non-empty groups, no data representation can guarantee that every classifier expressible under that representation is DP fair for all possible probability distributions over that domain. For fairness notions that take ground truth classification into account, like Odd Equality (EO),  given any two different non-redundant tasks\footnote{Namely, tasks in which the true label has some correlation with the group membership.}, no data representation can simultaneously allow accurate label classifiers for both while guaranteeing that any classifier expressible over that representation is EO fair for both these tasks. This impossibility applies even if one restricts the tasks in question to share the same  marginal (unlabeled) data distribution.

Our results answer negatively the main two open questions posed in the discussion section of Creager et al. \cite{DBLP:conf/icml/zemmel2019}.

There is an apparent discrepancy between our impossibility results and the long list of papers claiming to achieve fair representations. What is the source of that discrepancy? Note that there is a difference in the setup of the problem. The key distinguishing component is that in most (if not all) of the papers that claim  positive results about fair representations, the designer of the fair representation has access to the data distribution w.r.t. which the fairness is being evaluated. When the notion of fairness is independent of the ground truth classification (the case of Demographic Parity), the distribution in question is the marginal (unlabeled) one. When the notion of fairness of concern does involve true labels (such as Odds Equality or Group Calibration), the algorithms that define the representations require, on top of that, access to the ground truth labels of sample instances. What we show here is that this access to the data distribution  at evaluation (or test) time, is necessary for the ability to guarantee the fairness of representations. That common (often implicit) assumption can be justified only in very limited situations. For example, the definition of Demographic Parity for acceptance of students to a given university program depends on the distribution of applicants \textit{to that program at the given term}. This may change between universities, between programs and between academic years. Therefore, based on our results, any a priori designed data representation cannot be guaranteed to provide Demographic Parity fairness it aims to establish for acceptance of students to academic programs. The situation is  similar when it comes to granting loans - the distribution of applicants changes between loan granting institutions, branch locations,  requested sums, dates, etc. In fact, it is hard to come up with any realistic scenarios in which a fixed data distribution remains unchanged throughout various classification tasks that may use the data down the road. Therefore no data representation can meet the goal stated in  \cite{zemel2013learning}, namely - "be used in a blackbox fashion to turn any classification algorithm into a
fair classifier, by simply applying the classifier to the
sanitized representation of the data".

\subsection{What is \textit{fair representation}?}
The term `fair data representation' encompasses a wide range of different meanings. 
When word embedding results in smaller distance between the vectors representing `woman' and `nurse' relative to
the distance between the representations of `woman' and  `doctor' and the other way around for `man', is it an indication of bias in the \emph{representation} or is it just a faithful reflection of a bias in society? Rather than delving into  such issues, we discuss an arguably more concrete facet of data representation; We examine representation fairness from the perspective of its effect on the fairness of classification rules that agents using data represented that way may come up with. Such a view takes into consideration two setup characteristics: 
\begin{description}
    \item[The objective of the agent using the data] We distinguish three types of classification prediction agents (formal definitions of these aspects of fairness are provided in section \ref{agent-types}): 
    \begin{description}
    \item[\emph{ Malicious}] - driven by a bias against a group of subjects.
    To protect against such an agent, a fair representation (or feature set)  should be such that \textit{every} classifier based on data represented that way is fair. This is apparently the most common approach to fair representations in the literature e.g., \cite{zemel2013learning, pmlr-v80-madras18a}. 
    \item[\emph{ Accuracy Driven}]  - focusing on traditional measures of learning efficiency, ignoring fairness considerations.
    A representation is accuracy-driven fair if every loss minimizing classifier based on that representation is fair. 
    \item[\emph{Fairness Driven}] - aiming to find a decision rule that is fair while maintaining meaningful accuracy. A representation is fairness-driven fair if there exists a loss minimizing (or an approximate minimizer) classifier based on that representation that is fair.
    \end{description}
    
  \item[The notion of group fairness applied to the classification decisions] 
  The wide range of group fairness notions (for classification) can be taxonomized along several dimensions: Does the notion depend on the ground truth classification or only on the agent's decision (like demographic parity)? Is a perfectly accurate decision (matching the ground truth classification) always considered fair (like in odds equality)? Does the fairness notion depend on unobservable features (like intention or causality)? 
In this work we focus on fairness notions that are ground-truth-dependent, view the ground truth classification as fair and depend only on observable 
features.

Picking which notion of fairness one wishes to abide by depends on societal goals and may vary from one task to another. This is outside the scope of this paper. Just the same, let us briefly explain why the requirements listed above are  natural in many situations. 
\begin{description}
\item[\emph{The dependence on the ground truth classification}] 
is almost inevitable from a utilitarian perspective - taking into account the probability that a student succeeds or fails when making acceptance decisions should not be considered unfair. Put more formally, whenever there is any correlation between membership and the ground truth classification, any classifier that is fair w.r.t. a notion that ignored the ground truth (like demographic parity) is bound to suffer prediction error proportional to that correlation. 
\item[\emph{Viewing perfectly accurate decisions as fair}] can be viewed as a distinction between notions that do or do not try to inflict affirmative action. It makes a lot of sense in tasks like conviction in a crime - if you convict all criminals and no one else, you should not be accused of unfairness. 
\item[\emph{Relying only on observable features}] fosters objectivity and allows scrutiny of the decisions made.  Our running example of such a notion is odds equality \cite{hardt2016equality}, 
however our results hold as well for other  common notions of fairness that meet the above conditions (like Calibration within groups \cite{KleinbergMR16}). We provide formal definitions of these notions in Section \ref{subsec:groupfairness}.
\end{description}
\end{description}

\subsection{Our results}
We prove the following inherent limitations of notions of fair representations:
\begin{enumerate}
    \item \emph{The impossibility to be task-independent}. There is a host of literature proposing 
    methods for coming up with data representations that guarantees the fairness of classifiers based on those representations (e.g., \cite{DBLP:conf/icml/zemmel2019, pmlr-v80-madras18a, oneto2019learning}). We elaborate on some of these works in our Previous Work section. Contrasting the impression conveyed by many such papers, we show that the ability to guarantee multi-task fairness is inherently limited. Much of that work addresses Demographic parity (DP). We prove that if two tasks have different marginal data distributions (that is, the distribution of unlabeled instances), then no representation can guarantee that any non-trivial classifier trained on it satisfies DP for both. We show that the only classifiers that are guaranteed to satisfy any significant level of DP fairness w.r.t. all marginal distributions are the redundant constant functions. From a practical point of view, since DP fairness
    of some decision (say, acceptance to some university program) requires the ratio of positive decisions between groups to match the ratio of applicants from those groups, \textit{a representation that guarantees DP fairness cannot be a priori constructed} - it must have access to the distribution of groups among applicants for that specific program.
    Furthermore, we prove that for every fixed marginal data distribution, if two ground truth classifications differ with non-zero probability over it, \emph{there can be no data representation that enjoys Odds Equality fairness and accuracy with respect to both tasks over that shared marginal distribution} (except for the redundant case where the success rates of both groups are equal for both tasks). These results 
    answer negatively the main two open problems posed in the Discussion section of \cite{DBLP:conf/icml/zemmel2019}.
    
    \item \emph{The impossibility to evaluate the fairness contribution of a given feature devoid of the other features used} (again, for each agent objective and 
    several common group fairness notions). We show that for a fixed task, for each notion of fairness of representation, there are features that when added to one set of features render the resulting representation more fair, and when added to a different set of features render the resulting representation less fair. 
    \item \emph{The inherent dependence of the effect on fairness of adding/deleting a feature on the type of agent using the representation} (on top of the above mentioned dependence on other features), even when the feature in question does not correlate with membership in the protected group.
\end{enumerate} 
 (These come on top of the obvious dependence on the notion of fair classification sought).

\noindent{\bf Paper road map:}
Section 2 gives an overview of the related work. Section 3 introduces our setup including our taxonomy for fair representations. Section \ref{multi-task} contains our main results on the impossibility of generic fairness of a representation. Section \ref{feature-alone} addresses the impossibility of defining the fairness effect of a single feature without considering the other components of a representation. Section 6 briefly shows the impossibility of having fair representations w.r.t. Predictive Rate Parity. Section 7 is our concluding remarks.  

\section{Related Work}
 Most, if not all, of the literature concerning the creation of fair data representations addresses this task 
in a setup where some input data (or a probability distribution over some domain of individuals) is given to the agent building the representation (e.g., \cite{DBLP:journals/corr/EdwardsS15,pmlr-v80-madras18a,zemel2013learning, song2019learning}). Such a probability distribution is essential to any common definition of fairness. 
However, in many cases the probability distribution with respect to which the fairness is defined remains implicit. For example, \cite{zemel2013learning} define their notion of fairness by saying: "We formulate this using the notion of statistical
parity, which requires that the probability that a random element from $X^+$ maps to a particular prototype
is equal to the probability that a random element from
$X^-$ maps to the same prototype" (where  $X^+$ and $X^-$ are the two groups w.r.t. which one aims to respect fairness). However, they do not specify what is the meaning of "a random element". The natural interpretation of these terms is that "random" refers to the uniform distribution over the finite set of individuals over which the algorithm selects. In that case, that information varies with each concrete tasks and is not available to the task-independent representation designer. Alternatively, one could interpret those "random" selections as picking uniformly at random from some established large training set that is fixed for all tasks. Such randomness may well be available to the representation designer, but it misses the intention of statistical parity fairness; For example, the fixed training set may have 10,000 individuals from one group and 20,000 from the other group, but when some local bank branch allocates loans it has 80 applicants from the first group and 37 applicants from the other. For the fairness of these loan allocation decisions, the relevant ratio between the groups is 80/37 rather than the 10,000/20,000 ratio available to the representation designer.
 
 Almost all the work on fair representations focuses on the demographic parity (DP) notion of fairness \cite{DBLP:journals/corr/EdwardsS15,pmlr-v80-madras18a,zemel2013learning, song2019learning}. To achieve DP fairness, a classifier has to induce success ratio between the groups of subjects that match the ratio between these groups in the input data. However, as demonstrated above, that ratio varies from one application to another and cannot be determined a priori. We show that any fixed representation that allow expressing non-trivial classification cannot guarantee DP fairness in the face of shifting marginal (that is, unlabeled) data distribution (see section \ref{multi-task}). 
 
 When the data marginal distribution w.r.t. which the fairness is defined is fixed and available to the designer of a representation, then, as shown by \cite{zemel2013learning} and followup papers, 
 DP fairness is indeed possible. However, we further show that even under these assumptions, no data representation can guarantee fairness with respect to notions of fairness that do rely on the correct ground truth, such as equalized odds (EO) \cite{hardt2016equality},
 for arbitrary tasks (see Section \ref{multi-task}).
 
 To the best of our knowledge this fact also has not been explicitly stated (and proved) before, although it seems that some of the previous work were aware of this concern; in previous work discussing fair representation w.r.t. notions of fairness that take the ground truth classification into account, the algorithms that design the representations require access to task specific labeled data (e.g. \cite{Zhang_2018,DBLP:journals/corr/BeutelCZC17,song2019learning, CalmonWVRV17}). Such a requirement defies the goal of having a fixed representation that guarantees fairness for many tasks. 
 
The effect of the motivation of the user of the representation on the fairness of the resulting decision rule has been considered by Madras et al.\  \cite{pmlr-v80-madras18a} and Zhang et al.\ \cite{Zhang_2018}. These papers identify two motivations. The first is malicious, which is the intent to discriminate without regard for accuracy. The second is accuracy-driven, which is the intent to maximize accuracy. We address these effects as part of our taxonomy of notions of fair representations. Additionally, we discuss \textit{fairness-driven agents} that aim to achieve fairness while maintaining some level of accuracy.

The question of feature deletion has also been considered in real world examples, such as in the "ban the box" policy which disallowed employers using criminal history in hiring decisions \cite{doleac2016does}. The effect of allowing or disallowing features on fairness has been studied before, for example in Grgic-Hlaca et al.\ \cite{grgic2018beyond}. However in previous works, the effect of a feature on fairness, has been discussed in isolation. In contrast, we show that fairness of a feature should not be considered in isolation, but should also take into account the remaining features available. 

\section{Formal Setup} \label{sec:setup}

We consider a binary classification problem with label set $\{0,1\}$ over a domain $X$ of instances we wish to classify, e.g. individuals applying for a loan. 
We assume the task to be given by some distribution $P$ over $X\times \{0,1\}$ from which instances are sampled i.i.d. We denote the ground-truth labeling rule as $\Gt : X \rightarrow [0, 1]$. We will think of the label 1 as denoting `qualified' and the label 0 as `unqualified' and $\Gt(x)=\distr[y=1|x]$. For concreteness, we focus here on the case of deterministic labeling (that is $\Gt : X \rightarrow \{0, 1\}$). Most of our discussion can readily be extended to the probabilistic labeling case. In a slight abuse of notation we will sometimes use $t(w)$ to indicate the label coordinate of an instance $w\in X \times \{0,1\}$

A data representation is determined by a mapping $F: X \rightarrow Z$, for some set $Z$, and the learner only sees $F(x)$ for any instance $x$ (both in the training and the test/decision stages).  
We denote the hypothesis class of all feature based decision rules as $\mathcal{H}_{F}= \{h: Z \rightarrow \{0,1\}\}$.
As a loss function we consider a weighted sum of false positives and false negatives, i.e.
$$l^{\alpha}(h,x,y) = \begin{cases}

\alpha, & \mbox{ if } \: h(x) =0, y=1 \\
1-\alpha, & \mbox{ if } \:  h(x) =1, y=0 \\
0, & \mbox{ otherwise }
\end{cases}$$ 

for some weight $\alpha\in (0,1)$. We denote the true risk with such $\alpha$ weighted loss as $L_P^{\alpha}$ and the empirical risk, with respect to a training sample $S$, as $L_S^{\alpha}$. In this version of the paper we focus on the case of equal weights to both types of errors and use $L_P$ and $L_S$ to denote 
$L^{0.5}_P$ and $L^{0.5}_S$.

 \subsection{Notions of group fairness}\label{subsec:groupfairness}

For our fairness analysis we assume the population $X$ to be partitioned into two sub-populations $A$ and $D$
(namely, we restrict our discussion the case of one binary protected attribute). We sometimes use a function notation $G: X \rightarrow \{A,D\}$ to indicate the group-membership of an instance. Of course in reality there are often many protected attributes with more than two values. 
However, as our goal is to show limitations and impossibility results for fair representation learning, it suffices to only consider one binary protected attribute -- the same impossibilities readily follow for the more complex settings.

We now define two widely used notions of group-fairness that we will refer to throughout the paper, namely, equalized odds and demographic parity. In the following we will denote with $X_{g,l}$ the subset of $X$ with label $l$ and group membership $g$, i.e. $X_{g,l} = X\cap t^{-1}(l)\cap G^{-1}(g).$

The notion of group-fairness we will focus on in this paper is the ground-truth-dependent notion of odds equality as introduced by \cite{hardt2016equality}. 
\begin{defn}[Group fairness; Equalized odds]
    A classifier $h$ is considered fair w.r.t.\ to odds equality ($\EOObjective$) and a distribution $P$ if for $x\sim P$ we have the statistical independence $h(x) \indep G(x) | \Gt(x)$. For $g \in \{A,D\}$ let the false positive rate and the false negative rate be defined as  $\mbox{FPR}_{g}(h,t,P) = \mathbb{P}_{x\sim P}[h(x) = 1|x \in X_{g,0} ]$ and $\mbox{FNR}_{g}(h,t,P) = \mathbb{P}_{x\sim P}[h(x) = 0|x \in X_{g,1} ]$ respectively. The EO unfairness is given then by the sum of differences in false positive rate and false negative rate between groups:\\
     $$ \EOObj(h) = \frac{1}{2} |\mbox{FNR}_A - \mbox{FNR}_D| + \frac{1}{2} |\mbox{FPR}_A - \mbox{FPR}_D| .$$
     \end{defn}
If we say a classifier is fair, without referring to any particular group-fairness notion, we mean fairness w.r.t.\ equalized odds.

\begin{defn}[Demographic parity]
A classifier $h$ is considered fair w.r.t.\ to demographic parity ($\DPObjective$) and a distribution $P$ if for $x\sim P$, we have $h(x) \indep G(x)$. The respective unfairness is given by difference in positive classification rates between groups\\
$$\DPObj(h) =  |\mathbb{P}_{x\sim P}[h(x) =1 | G(x) = A] - \mathbb{P}_{x\sim P}[h(x) =1 | G(x) = D] |$$.
\end{defn}

\subsection{The role of the agent's objective} \label{agent-types}
We will phrase our definitions of representation fairness in terms of a general group fairness notion $\FairObjective$ with unfairness measure $\FairObj$. 

We start by considering a \emph{malicious decision maker} who tries to actively discriminate against one group. To protect against this kind of decision maker, we need to give a guarantee such that based on the feature set it is not possible to discriminate against one group. This corresponds to the notion of  adversarial fairness. 

\begin{defn}[Adversarial fairness]
A representation $F$ is considered to be \emph{adversarial} fair w.r.t.\ the distribution $P$ and group fairness objective $\FairObjective$ , if every classifier $h\in \mathcal{H}_{F}$ is group-fair. 
We define the adversarial unfairness of a representation $F$
by $U_{\mbox{adv}}(F) = \max_{h\in \mathcal{H}_{F}} \FairObj(h)$.
\end{defn}
Furthermore, we consider an \emph{accuracy-driven decision maker}, who aims to label instances correctly and is agnostic about fairness. For this kind of decision maker, we only need to make sure that optimizing for correct classification results in a fair classifier. The following definition ensures that the Bayes optimal classifier for a representation is fair.

\begin{defn}[Accuracy-driven fairness]
A representation $F$ is considered to be \emph{accuracy-driven} fair w.r.t.\ the fairness objective $\FairObjective$ and distribution $P$, and a threshold $\alpha\in (0,1)$, if every classifier $h\in \mathcal{H}_{F}$ with $ L_P^{\alpha}(h) = \min_{h\in \mathcal{H}_{F}} L_P^{\alpha}(h) $ is group-fair with respect to this objective. The accuracy-driven unfairness for a particular threshold parameter $\alpha$ is given by $U_{\mbox{acc}}^{\alpha}(\mathcal{F}) = \max\{\FairObj(h) : h\in \arg\min_{h\in \mathcal{H}_{F}} L_P^{\alpha}(h)  \}$.
\end{defn}

We note that in cases where the decision maker does not have access to the distribution $P$, but only to a labelled sample, this requirement might not be sufficient for guaranteeing that an accuracy-driven decision maker arrives at a fair decision. 

Lastly, we also consider a \emph{fairness-driven decision maker} who actively tries to find a fair and accurate decision rule, while maintaining some accuracy guarantees. For such a decision maker a representation should allow for fair and accurate decision rules. If a representation fulfills this requirement, we call it fairness-enabling.
\begin{defn}[$(\epsilon,\eta)$-fairness-enabling representation] 
A representation $F$ is considered to be \emph{$(\epsilon, \eta)$-fairness-enabling} w.r.t.\ a fairness objective $\FairObjective$, if there exists a classifier $h\in \mathcal{H}_{F}$ that such that $L_{P}^{\alpha}(h) \leq \epsilon$ and $\FairObj(h) \leq \eta$.
\end{defn}

Our discussion focuses primarily on the case of malicious and accuracy-driven decision makers. These notions of fair representation can be defined with respect to any group-fairness notion. In our paper we will mainly focus on the equalized odds notion of fairness \cite{hardt2016equality}. 

\section{Can there be a generic fair representation?} \label{multi-task}
We address the existence of a multi-task fair representation.
We prove that for the adversarial agent scenario (which is the setup that most fairness representation previous work is concerned with), \emph{\bf it is impossible to have generic non-trivial fair representations} - no useful representation can guarantee fairness for all "downstream" classification that are based on that representation (even if the ground truth classification remains unchanged and only the marginal may change between tasks).  

We start by considering scenarios in which only the marginals shift between two tasks, e.g. two openings for different jobs, requiring similar skills, for which different pools of people would apply. Such a distribution shift can likely affect one group more than another and would thus affect the classification rates of both groups differently. We show that we cannot guarantee fairness of a fixed data presentation for general shifts of this kind, even for the simplest case of demographic parity.

\begin{clm} \label{clm:marginalDP}
Pick any domain set $X$ and any partition of $X$ into non-empty subsets $A,D$. For every non-constant function
$f: X \to \{0,1\}$ there exists a probability distribution $P$ over $X$ such that $f$ is arbitrarily DP-unfair w.r.t. $P$ (say, $L_P^{\mbox{DP}}(h)> 0.9$). 

\end{clm}

In particular, for an agent that makes some non-trivial binary valued decision over a set of individuals divided to Advantaged ($A$) and disadvantaged ($D$), there will always be a probability distribution over the set of individuals (or, a subset of that set with the uniform distribution over it) relative to which that decision will be acutely Demographic Parity unfair. In other words, any representation that allows a non-constant classifier can not provide a DP fairness guarantee for all possible tasks over the same set of individuals.

\begin{proof}
\textit{If $f$ is constant on any of the groups $A$ or $D$ then, since $f$ is not a constant over $X$ there are points in the other group on which $f$ has the opposite value. Thus, from $f$ not being constant, we can conclude that there are two labels $y_1\neq y_2$, such that the sets $\{x\in A: f(x)=y_1\}$ and $\{x\in D: f(x)=y_2\}$ are both non-empty. Now we choose the marginal $P_X$ to assign probability $0.5$ to $\{x\in A: f(x)=y_1\}$ and probability 0.5 to $\{x\in D: f(x)=y_2\}$.
Clearly $f$ fails DP w.r.t. this $P$.}
\end{proof}

\begin{cor} \label{cor:marginalDP}
No data representation can guarantee the DP fairness of any non-trivial classifier w.r.t. all possible data generating distributions (over any fixed domain set with any fixed partition into non-empty groups). That is, any non-constant representation F, cannot be adversarially fair with respect to $\DPObjective$ and any arbitrary task $P$.

\end{cor}

\begin{proofof}{Corollary~\ref{cor:marginalDP}}
\textit{For any non-constant function $f$, we have seen that there exists a marginal $P_X$ such that $f$ does not fulfill demographic parity with respect to $P_X$ (Claim~\ref{clm:marginalDP}). Now if a representation $F$ is non-constant, it allows some non-constant function using that representation. Thus no non-constant representation can fulfill adversarial demographic parity with respect to \emph{any} distribution $P$.}
\end{proofof}

We can now show a similar effect for EO-fairness, i.e. we show that there is no representation that can guarantee EO fairness for arbitrary marginal shifts. This result is directly implied by the following claim.
\begin{clm} \label{clm:marginalEO}
For every function non-constant function $f: X\rightarrow \{0,1\}$ and every non-constant classifier $h:X\rightarrow \{0,1\}$ with $h\neq f$ and $h\neq 1-f$ (where $1$, denotes the function that maps every element to $1$), there exists a marginal $P_X$, such that $h$ has high unfairness with respect to $\EOObjective$ and $P=(P_X,f)$, (i.e. $\EOObj(h)\geq 0.5$).
\end{clm}
\begin{proofof}{Claim~\ref{clm:marginalEO}}
\textit{ Let $f: X \rightarrow \{0,1\}$ be any non-constant function and $h: X\rightarrow \{0,1\}$ be any non-constant classifier with $h\neq f, 1-f$. Then we know that at least three of the four sets $ \{x\in X: f(x)=1, h(x)=0\}$, $ \{x\in X: f(x)=0, h(x)=1\}$ , $ \{x\in X: f(x)=1, h(x)=1\}$ and $ \{x\in X: f(x)=0, h(x)=0\}$ are non-empty. 
Thus two of these three sets, agree on the ground truth. Call them $B_1$ and $B_2$ (and let the remaining set be $B_3$). W.l.o.g. $B_1 = \{ s\in X: f(x) =1, h(x) = 0 \}$, $B_2 = \{ s\in X: f(x) =1, h(x) = 1 \}$.
\begin{itemize}
    \item Case 1: $B_1 \cap A \neq \emptyset$ and $B_2 \cap D \neq \emptyset$. Then we can choose the marginal $P_X$ as $P_X(B_1 \cap A) =0.5$ and $P_X(B_2 \cap D) =0.5$. Yielding, $\EOObj(h) =0.5$
    \item Case 2: $B_2 \cap A \neq \emptyset$ and $B_1 \cap D \neq \emptyset$: Analogous to Case 1
    \item Case 3: there is $G\in \{A,D\}$, such that $B_1 \cap G = B_2 \cap G = \emptyset$. W.l.o.g. $G=A$. Then $B_3 \cap A \neq \emptyset$ and $B_1 \cap D \neq \emptyset $ and $B_2\cap D \neq \emptyset$. In this case we can choose the marginal as $P_X(A\cap B_3)= 0.5$ and $P_X(D\cap B_1)= 0.5$. Then all elements of $D$ will be misclassified and all elements of $A$ will either be classified correctly or be misclassified in the opposite direction, yielding to high EO unfairness. (In the case where the ground truth labeling is constant on one group, we define the misclassification rate with respect to the label it will not achieve to be zero. Then we get $\EOObj(h) \geq 0.5$.)
\end{itemize}
}
\end{proofof}

\begin{cor} \label{cor:marginalEO}
No data representation can guarantee EO fairness of any non-constant predictor based on that representation for all "downstream" classification learning tasks. That is, any representation F that is not constant on any group, cannot be adversarially fair with respect to $\EOObjective$ and any arbitrary task $P$. This holds even if one restricts the claim to tasks sharing a fixed marginal data distribution.
\end{cor} 
\begin{proofof}{Corollary~\ref{cor:marginalEO}} 
\textit{For any ground truth $f:X\rightarrow \{0,1\}$ and any representation $F: X\rightarrow Z$, that allows $h:Z \rightarrow \{0,1\}$ as described in Claim~\ref{clm:marginalEO}, there exists a marginal $P_X$ such that $h$ is highly EO unfair with respect to $(P_X,f)$. Note that as long as $h$ is not constant on either group, we can find $P$, such that the requirements from Claim~\ref{clm:marginalEO} are fulfilled. Thus the representation is not adversarially fair with respect to $(P_X,f)$ and $\EOObjective$. Thus any sufficiently complex representation cannot guarantee fairness for \emph{every} possible covariate shift.}
\end{proofof}

The results above showed that there is no representation that can guarantee fairness for an arbitrary task. But what happens if we limit our discussion to a predefined selection of tasks? We will show that even in this restricted case, there can be no representation that guarantees EO fairness with respect to a general predefined selection of tasks. We say a distribution $P$ has \emph{equal success rates}, if both groups have the same conditional probability of label $1$, i.e. $P[t(x) =1 | x\in A] = P[t(x) =1 | x\in D]$. We will now state the main result of this section.
\begin{thm} \label{thm:multitask}
Let $P_1$ and $P_2$ be the distributions defining two different tasks with the same marginal $P_X = P_{1,X} =  P_{2,X} $ such that at least one of the tasks does not have equal success rates.
There can be no data representation $F$ such that for $P_1, P_2$, the following criteria simultaneously hold: 
\begin{enumerate}
    \item $F$ is adversarially fair w.r.t. $P_1$ and $EO$ 
    \item $F$ is adversarially fair w.r.t. $P_2$ and $EO$
    \item $F$ allows for perfect accuracy w.r.t.\ to $P_1 $ and $P_2$, i.e., there are $h_1, h_2$
    both expressible over the representation  $F$,
    such that $L_{P_1}(h_1) = L_{P_2}(h_2) = 0$.
\end{enumerate}
\end{thm}
In order to prove this theorem we use the following lemma.
\begin{lem} \label{lem:multitask}
Pick any set $X$ and a partition of $X$ into two non-empty (disjoint) sets $A$ and $D$. Let $P$ be any probability distribution over $X$ such that both $P(A) \neq 0$ and $P(D) \neq 0$. Let $f,g : X \mapsto \{0,1\}$ such that $P[\{x : f(x) \neq g(x) \} \neq 0$. If $f$ is a EO fair classification w.r.t. $g$ (as the labeling rule) and $g$ is a EO fair classification w.r.t. $f$ (as the labeling rule), then $P[f(x)=1| A] = P[f(x)=1|D]$ and $P[g(x)=1| A] = P[g(x)=1|D]$.
\end{lem}
\begin{proofof}{Lemma~\ref{lem:multitask}}
Consider the following four sets:
$S=\{x: f(x)=1, ~ g(x)=0\}$, $T=\{x: f(x)=1, ~ g(x)=1\}$, $U=\{x: f(x)=0, ~ g(x)=1\}$, $V=\{x: f(x)=0, ~ g(x)=0\}$.\\

Let $S_A$, $T_A$, $U_A$, $V_A$, denote the intersections of these sets with the set $A$, (e.g., $S_A = S \cap A$), and similarly, $S_D$, $T_D$, $U_D$, $V_D$, denote the intersections of these sets with the set $D$. Notice that 
\begin{itemize}
\item $P[f(x)=1| A] = \frac{P(S_A) + P(T_A)}{P(A)}$.
\item $P[f(x)=1| D] = \frac{P(S_D) + P(T_D)}{P(D)}$.
\item $P[g(x)=1| A] = \frac{P(T_A) + P(U_A)}{P(A)}$.
\item $P[g(x)=1| D] = \frac{P(T_D) + P(U_D)}{P(D)}$.
\end{itemize} 
It follows that once one shows that each of these quantities can be expressed in terms of 
the false positive and false negative rates when each of $f$ or $g$ is considered the true classification and the other as the predicted labeling, then the conclusion of the lemma is implied by its EO assumptions. \\

Using the above notation, when $f$ is the true classification,\\

$\mbox{FPR}_A(g,t,P) = \frac{P(U_A)}{P(V_A) + P(U_A)}$ and $\mbox{FNR}_A(g,t,P) = \frac{P(S_A)}{P(S_A) + P(T_A)}$ (and similarly for $D$).\\

And when the true classification is $g$,\\

$\mbox{FPR}_A(f,t,P) = \frac{P(S_A)}{P(V_A) + P(S_A)}$ and $\mbox{FNR}_A(f,t,P) = \frac{P(U_A)}{P(U_A) + P(T_A)}$ (and similarly for $D$).\\

We will start with the case where all eight sets $U_A, V_A, S_A,T_A$ and $U_D, V_D, S_D,T_D$ are non-empty. We note, that in this case equalized false positive rates and false negative rates of $f$ with respect to $g$ gives us the following two equations,
$$\frac{P(U_A) }{P(V_A) + P(U_A) } = \frac{P(U_D) }{P(V_D) + P(U_D) }$$
and 
$$\frac{P(S_A) }{P(S_A) + P(T_A) } = \frac{P(S_D) }{P(S_D) + P(T_D) }.$$
This implies that there are two constants $\beta_1,\beta_2$ with 
$P(U_A) = \beta_1 P(V_A) $ and $P(U_D) = \beta_1 P(V_D) $ and $P(S_A) = \beta_2 P(T_A) $ and $P(S_D) = \beta_2 P(T_D) $.

Furthermore, $g$ being EO fair with respect to $f$ gives us
$$\frac{P(S_A) }{P(V_A) + P(S_A) } = \frac{P(S_D) }{P(V_D) + P(S_D) },$$
and 
$$\frac{P(U_A) }{P(U_A) + P(T_A) } = \frac{P(U_D) }{P(U_D) + P(T_D) }.$$
This implies that there is a constant $\beta_3$ such that $P(V_A) = \beta_3 P(S_A) $ and $P(V_D) = \beta_3 P(S_D) $.

Thus, 

$$P[f(x) =1 | A] = \frac{\beta_2 + 1 }{\beta_2 + 1 + \beta_3 \beta_2 (1 + \beta_1) } =  P[f(x) =1 | D],$$

and 

$$P[g(x) =1 | A] = \frac{1 + \beta_1 \beta_2 \beta_3 }{\beta_2 + 1 + \beta_3 \beta_2 (1 + \beta_1) } =  P[g(x) =1 | D].$$ The cases in which one or several of these sets are empty can be shown in an analogous way. This proves our claim.
\end{proofof}

Now we can prove our theorem.

\begin{proofof}{Theorem~\ref{thm:multitask}}
\textit{We note that in order for $F$ being adversarially EO fair with respect to both $P_1$ and $P_2$, both $h_1$ and $h_2$ need to be EO fair with respect to $P_1$ and $P_2$, from Lemma~\ref{lem:multitask}, we know that this implies that either $P_1=P_2$ or that both $P_1$ and $P_2$ have equal success rates. This proves our claim.}
\end{proofof}

\section{Fairness of a feature set vs. fairness of a feature} \label{feature-alone}

In this section we discuss feature deletion and its impact on the fairness of a representation. For this we assume our representation $F$ to consist of finitely many features $f_i:X\rightarrow Y_i$ i.e. for every $x\in X: F(x) = (f_1(x),\dots,f_n(x))$ and $Z = Y_1 \times \dots \times Y_n$. We limit our discussion to cases where all $Y_i$ are finite. While this assumption facilitates our analysis, we do not expect our results to be different in the cases of continuous features. We will denote the set of features as $\mathcal{F} = \{f_1,\dots,f_n\}$. Unless otherwise stated, we focus on the equalized Odds (EO) notion of group fairness. We denote by $U_{adv}(\mathcal{F})$ and $U_{acc}^{\alpha}(\mathcal{F})$ the adversarial and accuracy-driven EO fairness of the representation induced by the feature set $\mathcal{F}$ respectively.
We show that it is in general not possible to determine the effect a single feature has on the fairness of a representation without considering the full representation. This is the case even if our considered feature is not correlated with the protected attribute. 

\subsection{Opposing effects of a feature for accuracy-driven fairness of a representation}

We start our discussion with accuracy-driven fairness w.r.t.\ equalized odds. In this case we show that the deletion of a feature $f$ can lead to an increase in accuracy-driven unfairness for some set of other given features $\mathcal{F}$ and that the deletion of the \emph{same} feature $f$ can lead to a decrease in accuracy-driven unfairness for another set of other available features $\mathcal{F}'$. This implies that the fairness of the feature $f$ cannot be evaluated without context.
We show that this  phenomena holds for a general class of features that satisfy some non-triviality properties (That on the one hand do not reveal too much information about group membership and labels (non-committing), and on the other hand does not reveal identity when label and group information is given ($k$-anonymity \cite{Sweeney02kanonymity})). We will start by stating the non-triviality requirements for our theorem.
\subsubsection*{Non-Triviality properties}\label{def:nontriviality}
\begin{defn} We define the following two non-triviality requirements for a feature:
\begin{enumerate}

    \item {\bf Non-committing} We will call a feature \emph{non-committing} if it leaves some ambiguity about label and group membership. That is, a feature $f$ is non-committing if there are two distinct values $y_1$ and $y_2$, such that $f$ assigns each of these values to at least one instance of each $X_{A,0},X_{A,1},X_{D,1},X_{D,0}$. i.e. $ f^{-1}(y_1)\cap X_i \neq \emptyset$ and  $f^{-1}(y_2)\cap X_i\neq \emptyset$ for every $X_i \in \{X_{A,0},X_{A,1},X_{D,1},X_{D,0}\} $

\item {\bf $k$-anonymity} A feature $f$  is \emph{$k$-anonymous} if knowing this feature, group-membership and label, will only reveal identity of an individuals up to a set of at least $k$ individuals. Namely, for every combination of value of this feature, group membership and class label,
there are either no instances satisfying this combination or there are at least $k$ many such instances.
\end{enumerate}
\end{defn}

\begin{thm}{(Context-relevance for fairness of features)}\label{thm:deletion}
For every $2$-anonymous non-committing feature $f$, there exists a probability function $P$ over $X$ and feature sets $\mathcal{F}$ and $\mathcal{F}'$ such that: 
\begin{itemize}
\item The accuracy-driven fairness w.r.t $\EOObjective$, $P$ and $\alpha=0.5$ of $\featureset\cup \{f\}$ is greater than that of $\featureset$, i.e.
    \[ U_{\mbox{acc}}^{\alpha}(\featureset \cup \{f\}) < U_{\mbox{acc}}^{\alpha}(\featureset)\]
    Thus, deleting $f$ in this context will increase unfairness.
   \item The accuracy-driven fairness w.r.t $\EOObjective$, $P$ and $\alpha = 0.5$ of $\featureset'\cup \{f\}$ is less than that of $\featureset'$, i.e.
    \[ U_{\mbox{acc}}^{\alpha}(\featureset' \cup \{f\}) > U_{\mbox{acc}}^{\alpha}(\featureset')\]
    Thus, deleting $f$ in this context will decrease unfairness.
\end{itemize}

\end{thm}

We note that this phenomenon can occur for quite general pairs $(f,P)$ and that we mainly need to exclude pathological cases for our construction to work.\\
In particular we want to note that this phenomenon can occur even if $f$ is uncorrelated with the group membership and the label for ground-truth distribution $P$. We will give an example illustrating our last point and will refer the reader for the proof and a general discussion of the requirements on $(f,P)$ to the appendix.

Before giving our example, we need to introduce some concepts. 
\begin{description}
\item[feature-induced cells]
A set of features $\mathcal{F}= \{f_1,\dots,f_n\}$ induces an equivalence relation $\sim_{\mathcal{F}}$, by $x\sim_{\mathcal{F}}$ iff $f_i(x)=f_i(y)$ for all $i=1,\dots,n$. We call the equivalence classes with respect to $\sim_{\mathcal{F}}$ cells and denote the set of cells for a featureset $\featureset$ as $\cells$.
\item[ground-truth score function]We define the \emph{ground truth score function} $\score : \cells \rightarrow [0,1]$. $s_{\Gt}^P(C)$  is the probability, w.r.t.\ $P$, of $x \in C$ having the true-label $1$, i.e., $ s_t^P(C) = \ex_{x \sim P} [\Gt(x) | x \in C]$.
In cases where the distribution is unambiguous we will use the abbreviated notation $\score$ instead of $s_{\Gt}^P$.
\item[Bayes-optimal predictor] The predictor in $\Hcal_{F}$ that minimizes $L_P^\alpha$ is the \emph{Bayes Optimal predictor} $\Bopt$ that for a cell $C \in \cells$ assigns the label 1 if $\score(C) > \alpha$ and 0 otherwise.
\end{description}

We will now give an example in which both $f$ and $\featureset$ are adversarially fair w.r.t.\ $P$ and in which the phenomenon from Theorem~\ref{thm:deletion} holds:

Let the domain $X=\{x_1,x_2,x_3,x_4,x_5,x_6,x_7,x_8,x_9,x_{10},x_{11},x_{12}\}$ with $X_{A,1} = \{x_1,x_2,x_3\},X_{D,1}=\{x_4,x_5,x_6\}, X_{A,0}=\{x_7,x_8,x_9\}$, and $X_{D,0}=\{x_{10},x_{11},x_{12}\}$. Furthermore consider the uniform distribution $P$ over $X$, i.e. $P(\{x\})= \frac{1}{12}$ for every $x\in X$. 
 For the construction of the feature set, we only consider binary features $f_i: X \rightarrow \{0,1\}$. Now let $f$ be defined by $f^{-1}(1) = \{x_1,x_5,x_8,x_{12}\} $. Furthermore, let
$\featureset = \{f_1,f_2,f_3\}$ and $\featureset' =\{f'_1,f'_2\}$ with $f_1^{-1}(1) = \{x_1,x_2,x_3,x_5,x_8,x_{12}\}$, $f_2^{-1}(1) = \{x_1,x_2,x_3,x_5,x_{11},x_{12}\}$, $f_3^{-1}(1) = \{x_1,x_4x_5,x_6,x_7,x_{11}\}$,${f'}_1^{-1}(1)= \{x_1,x_4,x_7,x_{10}\}$ and ${f'}_2^{-1}(1)= \{x_1,x_2,x_4,x_5,x_7,x_8,x_{10},x_{11}\}$. The resulting cells for $\featureset$ and $\featureset'$ are $\mathcal{C}_{\featureset}= \{\{x_1,x_5\},\{x_2,x_3,x_{12}\}, \{x_8\},\{x_4,x_6,x_7\},\{x_9\}, \{x_{10},x_{11}\}\}$ and $\mathcal{C}_{\featureset'} = \{\{x_1,x_4,x_7,x_{10}\},\{x_2,x_5,x_8,x_{11}\},\{x_3,x_6,x_9,x_{12}\}\}$. It is easy to see that $\featureset'$ and $\{f\}$ are adversarially fair w.r.t.\ $P$ and $\EOObjective$. Furthermore, we have:
\[ U_{\mbox{acc}}^{0.5}(\featureset \cup \{f\}) = \frac{1}{2}|\frac{3}{3} - \frac{2}{3}| + \frac{1}{2}|\frac{2}{3} - \frac{1}{3}| = \frac{1}{3} > 0 = U_{\mbox{acc}}^{0.5}(\featureset)\]
and
\[ U_{\mbox{acc}}^{0.5}(\featureset' \cup \{f\}) = \frac{1}{2} |\frac{3}{3} - \frac{3}{3}| + \frac{1}{2} |\frac{1}{3} - \frac{1}{3}| = 0 < \]
\[\frac{1}{6} = \frac{1}{2}|\frac{3}{3} - \frac{3}{3}| + \frac{1}{2} |\frac{1}{3} - \frac{0}{3}| = U_{\mbox{acc}}^{0.5}(\featureset').\]
Thus we see that there are indeed features $f$ which are adversarially fair w.r.t. $P$ and equalized odds, for which there is this opposing effect of feature deletion.

\subsection{The fairness of a feature dependence on agent's objective}
We will now briefly discuss the effect of a single feature on fairness for the cases of an adversarial agent or a fairness-driven agent. In contrast to the accuracy-driven case, adding features has a monotone effect on the fairness of a fairness-driven and the malicious decision maker. We show in Theorem~\ref{thm:featuredeletion:worstcase} that adding any feature in the adversarial case, will only give the decision maker more information and thus give the decision maker more chances of discrimination. Similarly in the fairness driven case, any feature will only give the decision maker another option for fair decision making (Theorem~\ref{thm:featuredeletion:fairnessdriven}).  However, the quantitative effect of adding a feature on the unfairness can still range from having no effect to achieving perfect fairness/unfairness for both the fairness-driven and the malicious case. As in the accuracy-driven case, we will show (Theorem~\ref{thm:featuredeletion:fairnessdriven} and Theorem~\ref{thm:featuredeletion:worstcase}) that it is impossible to evaluate the quantitative effect of a feature on the fairness of a representation without considering the context of other available features.

\begin{thm}\label{thm:featuredeletion:worstcase}

\begin{enumerate}
    \item For any feature $f$ and any featureset $\mathcal{F}$ we have $U_{\mbox{adv}}(\mathcal{F}) \leq U_{\mbox{adv}}(\mathcal{F}\cup \{f\})$.
    \item For every distribution $P$ and feature $f$, there exists a feature set $\featureset$, such that adding $f$ will not impact the fairness of the distribution, e.g. $U_{\mbox{adv}}(\featureset) = U_{\mbox{adv}}(\featureset \cup \{f\})$.
    \item There exist distributions $P$, features $f$ and $\featureset'$, such that $U_{\mbox{adv}}(\mathcal{F}') = 0$  and $U_{\mbox{adv}}(\{f\}) = 0$, but $U_{\mbox{adv}}(\mathcal{F}'\cup \{f\})= 1$ . 
\end{enumerate}

\end{thm}
\begin{thm}\label{thm:featuredeletion:fairnessdriven}
\begin{enumerate}
\item For any feature $f$, if a representation $\mathcal{F}$ is $(\epsilon,\eta)$-fairness-enabling, the representation $\mathcal{F} \cup \{f\}$ is also $(\epsilon,\eta)$-fairness-enabling.

\item For every distribution $P$ and every feature $f$, there exists a feature set $\featureset$, such that $\mathcal{F}\cup \{f\}$ is $(\eta,\epsilon)$-fairness-enabling, if and only if $\featureset$ is $(\epsilon, \eta)$-fairness-enabling. Furthermore, there exists a distribution $P$, a feature $f$ and a feature set $\featureset'$, such that both $\featureset'$ and $\{f\}$ are not ($\epsilon,\eta$)-fairness-enabling for any $\epsilon,\eta< \frac{1}{2}$, but such that $\featureset' \cup \{f\}$ is ($0,0$)-fairness-enabling. 
\end{enumerate}
\end{thm}
While this section focused on fairness with respect to equalized odds, we note that many of these results can be replicated for other notions of fairness. In particular, analogous statements to Theorem~\ref{thm:featuredeletion:worstcase} and Theorem~\ref{thm:featuredeletion:fairnessdriven} can be made for demographic parity.

\section{Impossibility of adversarially fair representations with respect to predictive rate parity}

We now show that not all acceptable notions of group fairness always allow a adversarially fair representation, even in a single-task setting. One such notion is \emph{predictive rate parity}.

\begin{defn}{(Predictive rate parity (PRP))} A classifier $h$ is considered PRP fair 
w.r.t.\ to a distribution $P$ and if for $x\sim P$ the ground truth $\Gt(x)$ is statistically independent of the group membership $G(x)$, given the classification $h(x)$. We denote this fairness objective with $\PredObjective$.
\end{defn}

\begin{thm}\label{thm:predrate}
Adversarial fairness w.r.t.\ $P$ and $\PredObjective$ is only possible, if $P$ has equal success rates for both groups. 
\end{thm}
\begin{proofof}{Theorem~\ref{thm:predrate}}
We note that in order to achieve adversarial fairness with respect to any representation, the all-one classifier needs to be fair, as any representation $F$ admits any constant classifier. We furthermore note that the all-one classifier is fair with respect to predictive rate parity if and only if the ground truth has equal success rates. This shows our claim.

\end{proofof}

\section{Conclusion}
While many papers in this domain propose algorithmic solutions to fairness related issues, the main contributions of this paper are conceptual. We believe that, to a much larger extent than many other facets of machine learning, fundamental concepts of fairness in machine learning require better understanding.  Some basic questions are still far from being satisfactorily elucidated; What should be considered fair decision making? (various mutually incompatible notions have been proposed, but how to pick between them for a given real life application is far from being clarified). What is a fair data representation? To what extent should accuracy or other practical utilities be compromised for achieving fairness goals? and more.
The answers to these questions are not generic. They vary with the principles and the goals guiding the agents involved (decision makers, subjects of such a decision, policy regulators, etc.), as well as with what can be assumed regarding the underlying learning setup. We view these as the primary issues facing the field, deserving explicit research attention (in addition to the more commonly discussed algorithmic and optimization aspects). \\

Our main result addressed the existence of generic fair representations. We show that even label-independent fairness notions like demographic parity are vulnerable to shifts in marginals between tasks. For fairness notions that do rely on the true classification, we show that fairness and accuracy cannot be simultaneously achieved by the same data representation for any two different tasks even if they are defined over the same marginal (unlabeled) data distributions. We conclude the impossibility of having generic data representations that guarantee (even just) DP fairness with respect to tasks whose marginal distributions are not accessible when designing the representation. 
These insights stand in contrast to the impression arising from many recent papers 
\cite{pmlr-v80-madras18a, DBLP:journals/corr/EdwardsS15, mcnamara2019costs, song2019learning, DBLP:conf/icml/zemmel2019,  pmlr-v80-madras18a} 
that claim to learn transferable fairness-ensuring representations. 
We also considered the question of "fairness of a feature", which has been used in legal scenarios. We showed that the fairness of a single feature is an ill defined notion. Namely, the impact of a feature on the fairness of a decision cannot be determined without considering the other features of the representation\footnote{While we focused on the equalized odds notion of fairness, similar results can be shown for demographic parity (i.e. a feature that has demographic parity by itself can still make a representation demographic-parity unfair (in the adversarial sense) and for other common notions of group fairness. This is simply due to the fact that \emph{pairwise} statistical independence for a set of random variables does not imply statistical independence of the set of random variables.}

One obvious direction for further research is extending our impossibility results to quantitative accuracy-fairness trade-offs and bounds on what a data representation can guarantee over multiple tasks as a function of appropriate measures of task similarities.

 \bibliographystyle{plain}
\bibliography{references}

\newpage

\section*{Appendix}

\subsection*{Additional remarks on Theorem~\ref{thm:deletion}}

Theorem~\ref{thm:deletion} stated that every feature $f$ fulfilling some non-triviality requirements, there exists a distribution $P$ and feature sets $\mathcal{F}$ and $\mathcal{F}'$ such that adding $f$ to either of the feature sets has opposing effects on the accuracy-driven fairness of the respective representations. We will now state a condition on $f$ and $P$ for this phenomenon to occur. It will be easy to see that this condition is fulfilled for a very general class of distributions and features, only excluding pathological examples.

\begin{defn} \label{def:generic}

In the following let $l_1\in \{0,1\}$ denote a label and $G_1\in \{A,D\}$ a group. The opposing label and group will be denoted by $l_2$ and $G_2$ respectively. A pair $(f,P)$ of a feature $f$ and a distribution $P$ is called \emph{generic} if there exist sets $C_1,C_2,C_3\subset X$ with the following properties.
\begin{enumerate}
    \item $P(C_1) > P(C_2)$
    \item $C_1$ and $C_2$ are separated by the feature $f$, i.e. there are $y_1\neq y_2$ such that $C_1\subset f^{-1}(y_1)$ and $C_2\subset f^{-1}(y_2)$
    \item $C_1$ and $C_2$ are label-homogeneous for different labels and $C_2$ is group homogeneous, i.e. $C_1\subset t^{-1}(l_1)$ and $C_2 \subset X_{G_1,l_2}$.
    \item $C_3$ is not split by the feature, i.e. there is $y_3$ such that $C_3 \subset f^{-1}(y_3)$
    \item $C_3$ has the same majority label as $C_1$, i.e. $P(t^{-1}(l_1)\cap C_3 ) \geq P(t^{-1}(l_2)\cap C_3 )  $
    \item The fraction of elements of group $G_2$ and label $l_2$ in $C_3$ is sufficiently big in comparison to $C_2$, i.e. $\frac{P(C_3 \cap X_{G_2, l_2} )}{P(X_{G_2,l_2})} \geq \frac{P((C_2 \cup C_3)\cap X_{G_1,l_2})}{P(X_{G_1,l_2})} $.
\end{enumerate}
\end{defn}
\begin{lem} \label{lem:genericimpliesdeletion}
For every pair generic feature-distribution pair $(f,P)$, there are two feature sets $\mathcal{F}$ and $\mathcal{F}'$
\begin{itemize}
\item The accuracy-driven fairness w.r.t $\EOObjective$, $P$ and $\alpha=0.5$ of $\featureset\cup \{f\}$ is greater than that of $\featureset$, i.e.
    \[ U_{\mbox{acc}}^{\alpha}(\featureset \cup \{f\}) < U_{\mbox{acc}}^{\alpha}(\featureset)\]
    Thus, deleting $f$ in this context will increase unfairness.
   \item
    The accuracy-driven fairness w.r.t $\EOObjective$, $P$ and $\alpha = 0.5$ of $\featureset'\cup \{f\}$ is less than that of $\featureset'$, i.e.
    \[ U_{\mbox{acc}}^{\alpha}(\featureset' \cup \{f\}) > U_{\mbox{acc}}^{\alpha}(\featureset')\]
    Thus, deleting $f$ in this context will decrease unfairness.
\end{itemize}
\end{lem}

\begin{proof}
We define $\featureset$ as a representation which separates everything but a cell $C' = C_1\cup C_2$ by labels. For such a representation $\featureset\cup \{f\}$ enables perfect accuracy and therefore perfect fairness. However $\featureset$ is constructed in a way such that thresholding at $0.5$ leads to unfair classification, as only elements of $X_{G_1,l_2}$ are misclassified. Furthermore we can define $\featureset'$ as a representation that separates all but two cells $C' = C_1\cup C_2$ and $C'' = C_3$ perfectly by labels. As the only misclassification of Bayes classifier $t_{P, \mathcal{F}'\cup \{f\}}^{0.5}$ occurs on $C_3$ and it labels $t_{P, \mathcal{F}'\cup \{f\}}^{0.5} =l$ it has unfairness $L_P^{EO} (t_{P, \mathcal{F}'\cup \{f\}}^{0.5}) = \frac{1}{2} | \frac{P(C_3 \cap X_{G_2,l_2} )}{P(X_{G_2,l_2})} - \frac{P(C_3 \cap X_{G_1,l_2} )}{P(X_{G_1,l_2})} |$.
Furthermore the only misclassification for the Bayes classifier $t_{P, \mathcal{F}'}^{0.5}$  occurs on  $C_2$ and $C_3$ which are both labeled as $l$, yielding the unfairness
$t_{P, \mathcal{F}'}^{0.5} =l$ it has unfairness $L_P^{EO} (t_{P, \mathcal{F}'}^{0.5}) = \frac{1}{2} | \frac{P(C_3 \cap X_{G_2,l_2} )}{P(X_{G_2,l_2})} - \frac{P((C_2 \cup C_3 )\cap X_{G_1,l_2} )}{P(X_{G_1,l_2})} |$.
As $\frac{P((C_2 \cup C_3 )\cap X_{G_1,l_2} )}{P(X_{G_1,l_2})}  >  \frac{P( C_3 \cap X_{G_1,l_2} )}{P(X_{G_1,l_2})}$, by property (6.) of Definition~\ref{def:generic}, we thus get $L_P^{EO} (t_{P, \mathcal{F}'\cup \{f\}}^{0.5}) > L_P^{EO} (t_{P, \mathcal{F}'}^{0.5})$, concluding our proof.
\end{proof}

we wil now see how the non-triviality criteria for a feature $f$ from Theorem~\ref{thm:deletion} imply the existence of a generic pair $(f,P)$.

\begin{lem} \label{lem:nontrivialimpliesgeneric}
For every non-committing, $2$-anonymous feature $f$, there exists a distribution $P$, such that the pair $(f,P)$ is generic.

\end{lem}
\begin{proof}
We need to show that it is possible to define three sets $C_1,C_2,C_3,C_4\subset X$ and a distribution $P$ such that the requirements of Definition~\ref{def:generic} are fulfilled. From the fact that $f$ is non-committing we know that there are $y_1,y_2$ such that none of the subsets $f^{-1}(y_1) \cap X_{i}$ and $f^{-1}(y_2) \cap X_{i}$ is empty for any $X_i\in \{X_{A,0},X_{A,1},X_{D,1},X_{D,0}\}$. We can thus define the non-empty set $B=f^{-1}(y_2) \cap X_{A,0}$. Furthermore, we know that $f$ is also $2$-anonymous and thus we can split $B$ further into two non-empty subsets $C_2$ and $C_4$. Furthermore, we can define $C_1$ and $C_3$ as disjoint non-empty subsets of $f^{-1}(y_1)$, such that $C_1\subset f^{-1}(y_1) \cap t^{-1}(1)$ and such that $C_3\cap X_i \neq \emptyset$ for any $X_i \in \{X_{A,0},X_{A,1},X_{D,1},X_{D,0}\}$. Thus the properties (2.), (3.) and (4.) of the non-generic definition are fulfilled for the sets $C_1,C_2,C_3$.

We can now choose $P$ to pick probability weights as follows:
\begin{itemize}
    \item $P(C_1) =0.2 $
    \item $P(C_2) = 0.1$
    \item $P(C_3 \cap t^{-1}(1)) = 0.3 $
    \item $P(C_3 \cap X_{D,0} ) =0.2 $
    \item $P(C_4) =  0.2$
\end{itemize}
\end{proof}
Clearly (1.) is fulfilled as $P(C_1) =0.2 > 0.1 =P(C_2)$. Furthermore (5.) is fulfilled as, $P(C_3 \cap t^{-1}(1)) = 0.3 > 0.2 =P(C_3 \cap X_{D,0} )  = P(C_3 \cap t^{-1}(0) ) $.
Lastly, (6.) is fulfilled as:
\[\frac{P(C_3\cap X_{D,0})}{P(X_{D,0})} = 1 < \frac{1}{3} = \frac{P(C_2\cap X_{A,0})}{P(X_{A,0})} \]

\section*{Proofs}

\begin{proofof}{Theorem~\ref{thm:deletion}}
The result follows directly from Lemma~\ref{lem:nontrivialimpliesgeneric} and Lemma~\ref{lem:genericimpliesdeletion}.
\end{proofof} 

\begin{proofof}{Theorem~\ref{thm:featuredeletion:worstcase}}
\begin{enumerate}
\item We note that $\mathcal{H}_{\mathcal{F}} \subset \mathcal{H}_{\mathcal{F}\cup \{f\}}$. Thus any  $\arg \min_{h\in \mathcal{H}_{\mathcal{F}} } \EOObj(h) \leq \arg \min_{h\in \mathcal{H}_{\mathcal{F}\cup \{f\}} } \EOObj(h) $, proving the inequality for adversarial fairness. 
    \item For any distribution $P$ and feature $f$ we can choose a representation $\mathcal{F}$ such that $\cells = \mathcal{C}_{\mathcal{F}\cup \{f\}}$. It is obvious that the fairness will not change between those representations.
    \item The following example establishes the second claim:
    
Consider the domain $X=\{x_1,x_2,x_3,x_4,x_5,x_6,x_7,x_8\}$ with $X_{A,1}=\{x_1,x_2\}$ ,$X_{D,1}=\{x_3,x_4\}$, $X_{A,0}=\{x_5,x_6\}$ and $X_{D,0}=\{x_7,x_8\}$,. Furthermore let $\mathcal{F} = \{f_1,f_2\}$ with $f_1^{-1}(1) = \{x_1,x_3,x_5,x_7\}$ and $f_2^{-1}(1) = \{x_1,x_4,x_5,x_8\}$. Furthermore let $P$ be uniform over $X$,i.e. $P(\{x_1\}) = P(\{x_2\}) = P(\{x_3\}) = P(\{x_4\}) = P(\{x_5\}) = P(\{x_6\}) = P(\{x_7\}) = P(\{x_8\})= 0.125$. Thus, we have adversarial fairness w.r.t.\ EO for both features, i.e.
\[ \frac{P(X_{A,1} \cap f_1^{-1}(1))}{P(\Gt^{-1}(1)\cap f_1^{-1}(1))} = \frac{P(\{x_1\})}{P(\{x_1,x_2\})} = 0.5 = \]
\[\frac{P(\{x_3\})}{P(\{x_3,x_4\})} = \frac{P(X_{D,1} \cap f_1^{-1}(1))}{P(\Gt^{-1}(1)\cap f_1^{-1}(1))}. \]
\[ \frac{P(X_{A,0} \cap f_1^{-1}(1))}{P(\Gt^{-1}(0)\cap f_1^{-1}(1))} =\frac{P(\{x_5\})}{P(\{x_5,x_6\})} = 0.5=\]
\[= \frac{P(\{x_7\})}{P(\{x_7,x_8\})} = \frac{P(X_{D,0} \cap f_1^{-1}(1))}{P(\Gt^{-1}(0)\cap f_1^{-1}(1))}. \]
\[ \frac{P(X_{A,1} \cap f_2^{-1}(1))}{P(\Gt^{-1}(1)\cap f_2^{-1}(1))} \frac{P(\{x_1\})}{P(\{x_1,x_2\})} = 0.5 =\]
\[= \frac{P(\{x_4\})}{P(\{x_3,x_4\})} = \frac{P(X_{D,1} \cap f_2^{-1}(1))}{P(\Gt^{-1}(1)\cap f_2^{-1}(1))}. \]
\[ \frac{P(X_{A,0} \cap f_2^{-1}(1))}{P(\Gt^{-1}(0)\cap f_2^{-1}(1))} \frac{P(\{x_5\})}{P(\{x_5,x_6\})} = 0.5 \]
\[= \frac{P(\{x_8\})}{P(\{x_7,x_8\})} =\frac{P(X_{D,0} \cap f_2^{-1}(1))}{P(\Gt^{-1}(0)\cap f_2^{-1}(1))}. \]
However, the featureset $\featureset$ does not have adversarial fairness w.r.t.\ EO: $\cells = \{C_1,C_2,C_3,C_4\}$ with $C_{1} = \{x_1,x_5\}$, $C_{2} = \{x_2,x_6\}$, $C_{3} = \{x_3,x_7\}$, and $C_{4} = \{x_4,x_8\}$. Consider the classifier $h\in\mathcal{H}_{\mathcal{F}}$ with $h^{-1}(1)= \{C_1,C_2 \} $. Then
$\EOObj(h) = \frac{1}{2} \sum_{l\in\{0,1\}}\left |\frac{P(h^{-1}(|1-l|) \cap X_{A,l})}{P(X_{A,l})} - \frac{P(h^{-1}(|1-l|) \cap X_{D,l})}{P(X_{D,l})} \right | = \frac{1}{2}(|1 - 0 | + |0 - 1|) = 1 $. Thus $U_{\mbox{adv}}^{EO}(\featureset) = 1$.
\end{enumerate}
\end{proofof}

\begin{proofof}{Theorem~\ref{thm:featuredeletion:fairnessdriven}}
\begin{enumerate}
\item We note that $\mathcal{H}_{\mathcal{F}} \subset \mathcal{H}_{\mathcal{F}\cup \{f\}}$. Any $h\in\mathcal{H}_{\mathcal{F}} $  with  $\epsilon$ loss and $\eta$ unfairness, is also an element of $h\in\mathcal{H}_{\mathcal{F}\cup \{f\}}$, proving our claim.

 \item Similarly to the proof of Theorem~\ref{thm:featuredeletion:worstcase}, given a feature $f$ and a distribution $P$,we can construct a feature set $\featureset$, such that $\cells = \mathcal{C}_{\featureset \cup \{f\}}$. Since this implies that $\mathcal{H}_{\mathcal{F}} = \mathcal{H}_{\mathcal{\featureset \cup \{f\}}}$, we get the same fairness-enabling for both distributions.\\
Furthermore we can construct the following example to proof the second claim:
Consider the domain $X=\{x_1,x_2,x_3,x_4,x_5,x_6,x_7,x_8\}$ with $X_{A,1}=\{x_1,x_2\}$ ,$X_{D,1}=\{x_3,x_4\}$, $X_{A,0}=\{x_5,x_6\}$ and $X_{D,0}=\{x_7,x_8\}$,. Furthermore let $\mathcal{F} = \{f_1,f_2\}$ with $f_1^{-1}(1) = \{x_1,x_3,x_5,x_7\}$ and $f_2^{-1}(1) = \{x_1,x_3,x_6,x_8\}$. Furthermore let $P$ be uniform over $X$,i.e. $P(\{x_1\}) = P(\{x_2\}) = P(\{x_3\}) = P(\{x_4\}) = P(\{x_5\}) = P(\{x_6\}) = P(\{x_7\}) = P(\{x_8\})= 0.125$.
On both $\{f_1\}$ and $\{f_2\}$ there are no classifier $h_1\in \mathcal{H}_{\mathcal{\{f_1\}}}$ or $h_2\in \mathcal{H}_{\mathcal{\{f_2\}}}$ with $L_P^{\alpha}(h_1) < \frac{1}{2}$ or $L_P^{\alpha}(h_2) < \frac{1}{2}$ respectively for any $\alpha\in (0,1)$. Therefore $\{f_1\}$ and $\{f_2\}$ are both not $(\epsilon,\eta)$-best case fair for any $\epsilon,\eta <\frac{1}{2}$.
Furthermore, the classifier $h_3$ defined by $h^{-1}(1) = \{x_1,x_2,x_3,x_4\}$ is element of $\mathcal{H}_{\{f_1\}\cup\{f_2\}}$ and has loss $L_P^{\alpha}(h) = 0$ and unfairness $\EOObj(h)  = 0$. Thus $\{f_1\}\cup \{f_2\}$ is $(0,0)$-fairness-enabling for any $\alpha\in (0,1)$.
\end{enumerate}
\end{proofof}

\end{document}